\documentclass[letterpaper, 10 pt, conference]{ieeeconf}
\IEEEoverridecommandlockouts
\usepackage{fullpage,graphicx,amssymb,dsfont} 
\usepackage[caption=false,font=footnotesize]{subfig}
\usepackage{enumerate}
\usepackage{overpic}
\usepackage{cite}
\graphicspath{{./pictures/pdf/},{./pictures/ps/},{./pictures/png/},{./pictures/jpg/}}
\usepackage{url}
\usepackage{hyperref}
\usepackage{amssymb} 
\hypersetup{
  colorlinks =true,
  citecolor = black,
  urlcolor = black,
  linkcolor = black
}

\def\captionstyle{}
\def\boxcaptionstyle{\raggedright}

\def\maxfigfraction{.6}

\newdimen\figboxmargin
\newdimen\figboxhang
\def\DVIscaling{1}
\def\globalscaling{1}
\def\figuredirectory{./figures}
\let\boxer=\llboxer

\def\missingfigure#1{\hbox{Missing figure #1.ps}}

\catcode `\@=11

\newbox\figurebox

\def\figbox{
\@ifnextchar[{\figboxaux}{\figboxaux[htb]}}

\long\def\figboxaux[#1#2]#3#4#5#6{
\writepict{{#3}{#4}{#5}{#6}}
\setbox\figurebox\hbox{#3}%
\if l#1\tryleftbox{#4}{#5}{#6}%
\else
\if r#1\tryrightbox{#4}{#5}{#6}%
\else
\if *#1\checktwocoloptions#2]{\box\figurebox}{#4*}{#5}{#6}%
\else\tryonecol[#1#2]{#4}{#5}{#6}%
\fi
\fi
\fi\ignorespaces}

\long\def\tryleftbox#1#2#3{
\ifdim\wd\figurebox>\maxfigfraction\columnwidth \tryonecol[htb]{#1}{#2}{#3}%
\else\leftbox{\captionbox{\box\figurebox}{#1}{#2}{#3}}\fi}

\long\def\tryrightbox#1#2#3{
\ifdim\wd\figurebox>\maxfigfraction\columnwidth \tryonecol[htb]{#1}{#2}{#3}%
\else\rightbox{\captionbox{\box\figurebox}{#1}{#2}{#3}}\fi}

\def\checktwocoloptions{
\@ifnextchar]{\floatbox[htb}{\floatbox[}}

\long\def\tryonecol[#1]#2#3#4{
\ifdim\wd\figurebox>\columnwidth \floatbox[#1]{\box\figurebox}{#2*}{#3}{#4}%
\else\floatbox[#1]{\box\figurebox}{#2}{#3}{#4}\fi}

\long\def\floatbox[#1]#2#3#4#5{%
\begin{#3}[#1]
\hbox to \hsize{\hfil#2\hfil}
\captionandlabel{#3}{#4}{#5}
\end{#3}
}

\long\def\captionbox#1#2#3#4{
\setbox\figurebox\hbox{#1}%
\parbox[t]{\wd\figurebox}{%
\bigskip\box\figurebox
\let\captionstyle=\boxcaptionstyle
\captionandlabel{#2}{#3}{#4}
\bigskip
}}

\def\captionandlabel#1#2#3{
\def\testit{#3}%
\ifx\testit\empty\else
\writecapt{{#1}{#2}{#3}}
\captypeunstarred#1*.
\getcaption#3\endc@ption
\def\testit{#2}
\ifx\testit\empty\else\label{#2}\fi
\fi}

\def\captypeunstarred#1*#2.{
\def\@captype{#1}}

\def\getcaption{\@ifnextchar[{\getcaptwo}{\getcapone}}
\long\def\getcapone#1\endc@ption{\caption[#1]%
{\def\baselinestretch{1}\Large\normalsize\captionstyle\ignorespaces #1}}
\long\def\getcaptwo[#1]#2\endc@ption{\caption[#1]%
{\def\baselinestretch{1}\Large\normalsize\captionstyle\ignorespaces #2}}

\newdimen\figboxht
\newcount\figboxn

\newcount\figboxlines
\newdimen\figboxwid

\newif\ifisleftbox

\long\def\leftbox#1{%
\setbox\figurebox\hbox{#1}\global\isleftboxtrue
\startmarginbox
\vadjust{\smash{\rlap{\hskip\hsize\hskip\figboxhang
\llap{\raise.7\baselineskip\box\figurebox\hskip\rightskip}}}}%
\endmarginbox%
}

\long\def\rightbox#1{%
\setbox\figurebox\hbox{#1}\global\isleftboxfalse
\startmarginbox
\smash{\llap{\raise.7\baselineskip\box\figurebox\hskip\figboxmargin}}%
\endmarginbox%
}

\def\startmarginbox{%
\ifvmode\passpict\let\endmarginbox=\indent
\else\message{WARNING: marginbox in not in vmode}\hfilneg\ \passpict
\let\endmarginbox=\relax\fi
\figboxht=\dp\figurebox
\advance\figboxht by 1.3\baselineskip
\vskip.95\figboxht\penalty-300\vskip-.95\figboxht
\divide\figboxht by\baselineskip
\global\figboxlines=\figboxht
\global\figboxwid=\wd\figurebox
\global\advance\figboxwid by \figboxmargin
\global\advance\figboxwid by -\figboxhang
\setmypar\noindent}
\def\addlines#1{\global\advance\figboxlines by #1\myparshape}
\def\zerolines{\origpar\global\figboxlines=0\myparshape}

\def\passpict{\par\ifnum\figboxlines>1\vskip\figboxlines\baselineskip
\zerolines\fi}

\def\emptybox#1#2{\hbox to #1{\vbox to #2{\vss}\hss}}

\global\let\origpar=\@@par
\global\let\dopar=\origpar
\global\def\@@par{\dopar}
\@setpar{\dopar}

\def\setmypar{\global\let\dopar=\mypar
\global\prevgraf=0\myparshape}

\def\mypar{\origpar\global\advance\figboxlines by -\prevgraf%
\global\prevgraf=0\myparshape}

\def\myparshape{\relax%
\ifnum\figboxlines>1\theparshape \else
\global\hangindent=0pt\global\hangafter=1
\global\let\dopar=\origpar\fi}

\def\theparshape{%
\ifisleftbox\global\hangindent=-\figboxwid 
\else\global\hangindent=\figboxwid \fi
\global\hangafter=-\figboxlines \global\advance\hangafter by 1%
}

\def\definefnum#1{
\def\fnum@figure{Figure \ref{#1}}%
\def\fnum@table{Table \ref{#1}}%
\def\fnum@code{Algorithm \ref{#1}}%
}

\def\writepict#1{}
\def\writecapt#1{}

\def\journalpicts#1{
\newwrite\pictfile
\newwrite\captfile
\openout\pictfile\jobname.pic
\openout\captfile\jobname.cap
\gdef\writepict##1{\unexpandedwrite\pictfile{\doit##1}}%
\gdef\writecapt##1{\unexpandedwrite\captfile{\doit##1}}%
\global\let\ENDdocument=\enddocument
\gdef\enddocument{\DOjournalpicts{#1}\ENDdocument}
}

\def\DOjournalpicts#1{{%
\def\writepict##1{}\closeout\pictfile
\def\writecapt##1{}\closeout\captfile
\@fileswfalse
\onecolumn
\def\globalscaling{#1}
\def\doit##1##2##3##4{
\figboxaux[t]{\hss##1\hss}{##2}{}{}%
\vspace*{1in}
\definefnum{##3}
\captionandlabel{##2}{}{##4}
\clearpage}%
\input\jobname.pic
\def\doit##1##2##3{
\definefnum{##2}
\captionandlabel{##1}{}{##3}}%
\raggedright\let\captionstyle=\raggedright
\def\@makecaption##1##2{##1: ##2\par}
\input\jobname.cap
}}

\def\llboxer#1{\vbox to \figboxht{\vfil\hbox to \figboxwid{#1\hfill}}}
\def\lcboxer#1{\vbox to \figboxht{\vfil\hbox to \figboxwid{\hfill#1\hfill}}}
\def\oldboxer#1{\vbox to \figboxht{\vfil
                      \hbox to \figboxwid{\hfill\llap{#1\hskip4.25in}\hfill}}}
\def\ulboxer#1{\vbox to \figboxht{\hbox to \figboxwid{#1\hfill}\vfil}}
\def\ccboxer#1{\vbox to \figboxht{\vfil
                        \hbox to \figboxwid{\hfill#1\hfill}\vfil}}

{\catcode`\p=12\catcode`\t=12
\gdef\removedimen#1pt{#1}}

\def\defscaled#1#2{#2=\DVIscaling#2%
\xdef#1{\expandafter\removedimen\the#2}}

\def\DVIspace{ }
\newdimen\hscalefactor
\newdimen\vscalefactor

\def\scale#1{\horizscale{#1}\vertscale{#1}}
\def\horizscale#1{\hscalefactor=#1\hscalefactor\figboxht=#1\figboxht}
\def\vertscale#1{\vscalefactor=#1\vscalefactor\figboxwid=#1\figboxwid}

\def\boxps{%
\@ifnextchar[{\boxpsaux}{\boxpsaux[\relax]}}

\def\boxpsaux[#1]#2#3#4#5{%
{\figboxwid#4\figboxht#5\hscalefactor=1pt\vscalefactor=1pt%
\scale{#3}%
\scale{\globalscaling}%
#1%
\defscaled\DVIhscale\hscalefactor
\defscaled\DVIvscale\vscalefactor
\boxer{\includegraphics{\figpsfilename\DVIspace}}}%
}

\newread\Epsffilein
\newif\ifEpsffileok
\newif\ifEpsfbbfound

\newdimen\pspoints
\divide\pspoints by 72

\def\boxeps{%
\@ifnextchar[{\boxepsaux}{\boxepsaux[\relax]}}
\def\boxepsaux[#1]#2#3{%
\gdef\figpsfilename{\figuredirectory/#2.ps}
\openin\Epsffilein=\figuredirectory/#2.ps
\ifeof\Epsffilein 
\gdef\figpsfilename{\figuredirectory/#2.eps}
\openin\Epsffilein=\figuredirectory/#2.eps \fi
\ifeof\Epsffilein\message{I couldn't open \figuredirectory/#2.ps or \figpsfilename}%
\missingfigure{#2}
\else
   {\Epsffileoktrue\Epsfbbfoundfalse
    \catcode`\%=11 \catcode`\\=11
    \catcode`\{=11 \catcode`\}=11
    \catcode`\$=11 \catcode`\^=11
    \catcode`\&=11 \catcode`\#=11
    \catcode`\~=11 \catcode`\_=11
    \loop
       \read\Epsffilein to \Epsffileline
       \ifeof\Epsffilein\Epsffileokfalse\else
          \expandafter\Epsfaux\Epsffileline . .\\%
       \fi
   \ifEpsffileok\repeat
   \ifEpsfbbfound
        \figboxht=\Epsfury\pspoints
        \advance\figboxht by-\Epsflly\pspoints
        \figboxwid=\Epsfurx\pspoints
        \advance\figboxwid by-\Epsfllx\pspoints
   \else
        \message{No bounding box comment in \figpsfilename }%
        \figboxwid=2in\figboxht=1in%
   \fi%
   \immediate\closein\Epsffilein
   \hscalefactor=1pt\vscalefactor=1pt%
   \scale{#3}%
   \scale{\globalscaling}%
   #1%
   \defscaled\DVIhscale\hscalefactor
   \defscaled\DVIvscale\vscalefactor
   \hscalefactor=-\Epsfllx\hscalefactor
   \hscalefactor=1.00375\hscalefactor
   \defscaled\DVIhoffset\hscalefactor
   \vscalefactor=-\Epsflly\vscalefactor
   \vscalefactor=1.00375\vscalefactor
   \defscaled\DVIvoffset\vscalefactor
   \llboxer{\includegraphics{\figpsfilename\DVIspace}} }%
\fi
}%
{\catcode`\%=11 \global\let\Epsfpar=\par
\global\let\Epsfpercent=
\long\def\Epsfaux#1#2 #3\\{\relax\ifx#1\Epsfpercent
   \def\testit{#2}\ifx\testit\Epsfbblit
      \Epsfsize #3 . . . .\\%
      \global\Epsffileokfalse
      \global\Epsfbbfoundtrue
   \fi\else\ifx#1\Epsfpar\else\global\Epsffileokfalse\fi\fi}%
\def\Epsfsize#1 #2 #3 #4 #5\\{\global\def\Epsfllx{#1}\global\def\Epsflly{#2}%
   \global\def\Epsfurx{#3}\global\def\Epsfury{#4}}%

\catcode`\@=12                  

\def\pic#1;#2;#3;#4\par{\picsc#1;#2;#3;1;#4\par}

\def\picsc#1;#2;#3;#4;#5\par{
\figbox[htb]{\boxeps{#1}{#4}
}{figure}{#1}{%
#5}}

\def\mpic#1;#2;#3;#4\par{\mpicsc#1;#2;#3;1;#4\par}

\def\mpicsc#1;#2;#3;#4;#5\par{
\figbox[l]{\boxeps{#1}{#4}
}{figure}{#1}{%
#5}}

\newtheorem{theorem}{Theorem}

\newtheorem{lemma}[theorem]{Lemma}

\newcommand{\old}[1]{{}}
\def\tomath#1{\relax\ifmmode#1\else$#1$\fi}

\long\def\comm#1{\ignorespaces}
\def\comments{\long\def\comm##1{\message{COMMENT: ##1}{\bf(( ##1 ))}}}

\comments

\newcommand{\figwid}{0.22\columnwidth}

\title{Particle Computation:\\ Designing Worlds to Control Robot Swarms with only Global Signals}
\date{}

\index{Becker, Aaron}
\index{Demaine, Erik D.}
\index{Fekete, S\'andor P.}
\index{McLurkin, James}

\author{
  Aaron Becker%
    \thanks{Department of Computer Science, Rice University, Houston, TX 77005,
     \protect\url{aabecker@gmail.com}, 
\protect\url{jm23@rice.edu}.}
\and
  Erik D. Demaine%
    \thanks{Computer Science and Artificial Intelligence Laboratory, MIT,
     Cambridge, MA 02139, USA,
      \protect\url{edemaine@mit.edu}.}
\and
  S\'andor P. Fekete%
    \thanks{Dept.~of Computer Science,
      TU Braunschweig,
      M\"uhlenpfordtstr.~23, 38106 Braunschweig, Germany,
      \protect\url{s.fekete@tu-bs.de}}
\and
 James McLurkin%
}

\begin{document}
\thispagestyle{empty}
\maketitle
\begin{abstract}
Micro- and nanorobots are often controlled by global input signals, such as an electromagnetic or gravitational field. These fields move each robot maximally until it hits a stationary obstacle or another stationary robot.  This paper investigates 2D motion-planning complexity for  large swarms of simple mobile robots (such as bacteria, sensors, or smart building material). 

In previous work we proved it is NP-hard to decide whether a given initial configuration can be transformed into a desired target configuration; in this paper we prove a stronger result: the problem of finding an optimal control sequence is PSPACE-complete. On the positive side, we show we can build useful systems by designing obstacles.  We present a reconfigurable hardware platform and demonstrate how to form arbitrary permutations and build a compact absolute encoder.  We then take the same platform and use \emph{dual-rail logic} to build a universal logic gate that concurrently evaluates AND, NAND, NOR and OR operations.  Using many of these gates and appropriate interconnects we can evaluate any logical expression.

\end{abstract}

\section{Introduction}\label{sec:Intro}

Milli-, micro-, and nanorobots are capable of entering environments too small for their larger cousins.   Swarms of these tiny robots may be ideal for targeted drug delivery, on-site micro construction, and minimally invasive surgery.  An untethered swarm could reach areas deep in the body that traditional,  robots and tooling cannot.    These swarms are often controlled by an external, global electromagnetic field~\cite{Chanu2008,Khalil2013b,Lanauze2013}. Motion planning for large robotic populations actuated by the same field in a tortuous environment is difficult.

We investigate the following basic problem: {\it Given a map of an environment, such as the vascular network shown in Fig.~\ref{fig:vascularNetwork}, along with initial and goal positions for each robot, does there exist a sequence of inputs that will bring each robot to its goal position?}
 In previous work \cite{Becker2014a}, it was shown that this problem is at
least NP-hard, by reduction to a 3SAT problem.  In this paper we improve the
analysis and show the problem is  PSPACE-complete.   This complexity result has
some benefits: we show that we can design artificial environments capable of
computation, and describe configurations of obstacles that result in useful
robotic systems: absolute encoders, Boolean logic as shown in Fig.~\ref{fig:ParticleLogic11}, and planar displays.

\begin{figure}
\begin{overpic}[height=2.75in]{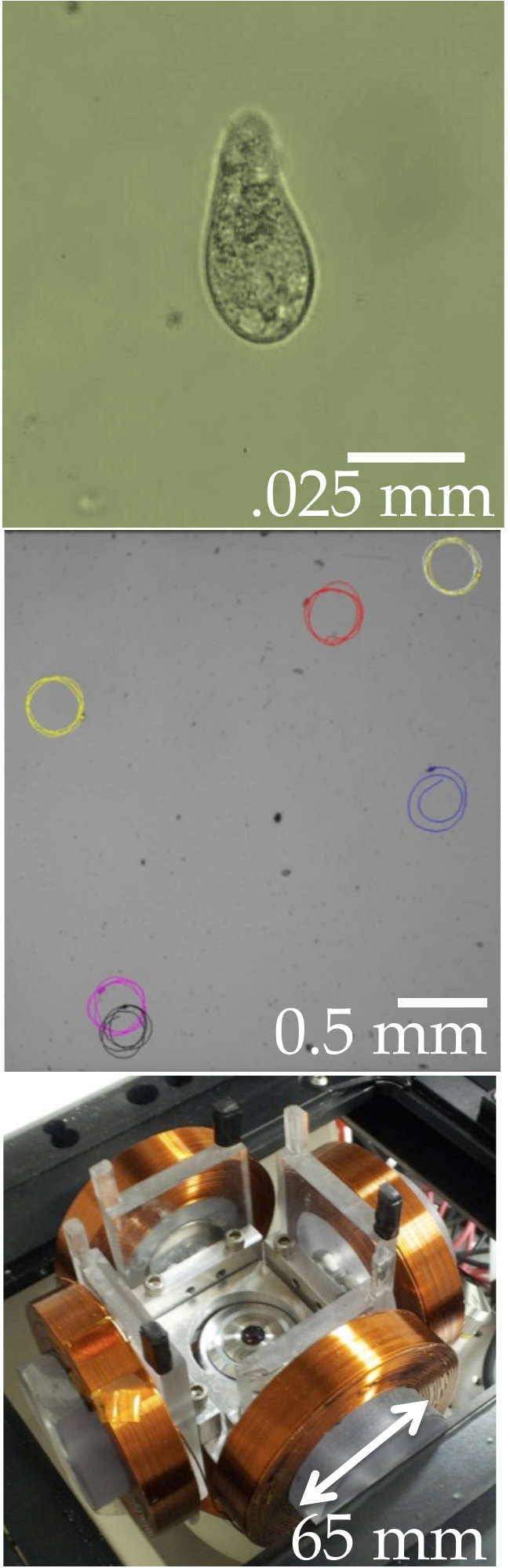}
\end{overpic}
\begin{overpic}[height=2.75in, angle = 0]{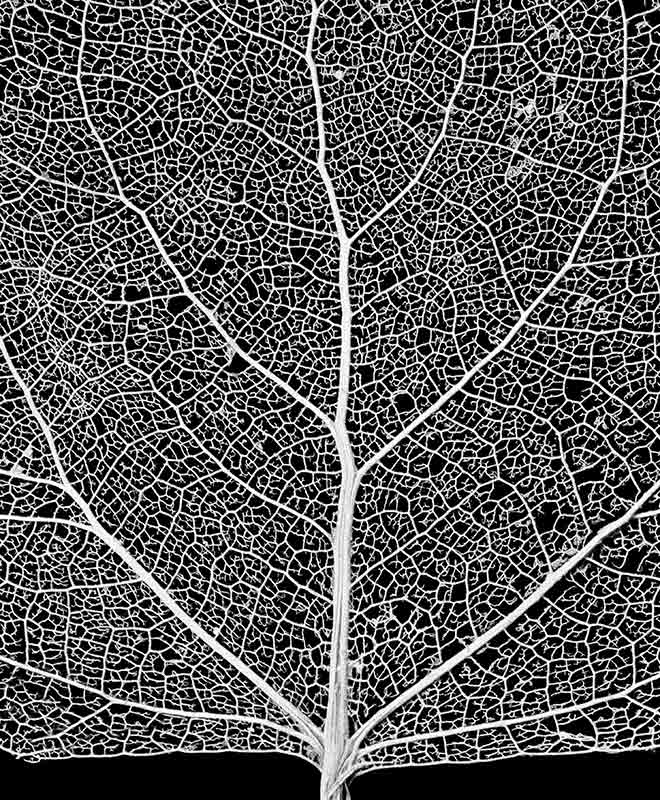}\end{overpic}
\caption{\label{fig:vascularNetwork}(Left) State of the art in controlling small objects by force fields: after feeding iron particles to \emph{T. pyriformis} cells and magnetizing the particles with a permanent magnet, the cells are mobile robots that can be turned by changing the orientation of an external magnetic field~\cite{Becker2013a}.  All cells are steered by the same global field.
\href{http://www.mathworks.com/matlabcentral/fileexchange/42890-simulate-control-of-magnetized-tetrahymena-pyriformis-cells}{(Right) A complex vascular network,
forming a typical environment for the parallel control of small robots.   Given such a network along with initial and goal positions of $N$ robots,  is it possible to bring each robot to its goal position using a global control signal?\hspace{\textwidth}
 (Right image credit: Royce Bair/Flikr/Getty Images)}
}
\vspace{-1em}
\end{figure}

\begin{figure}
\begin{overpic}[width =\columnwidth]{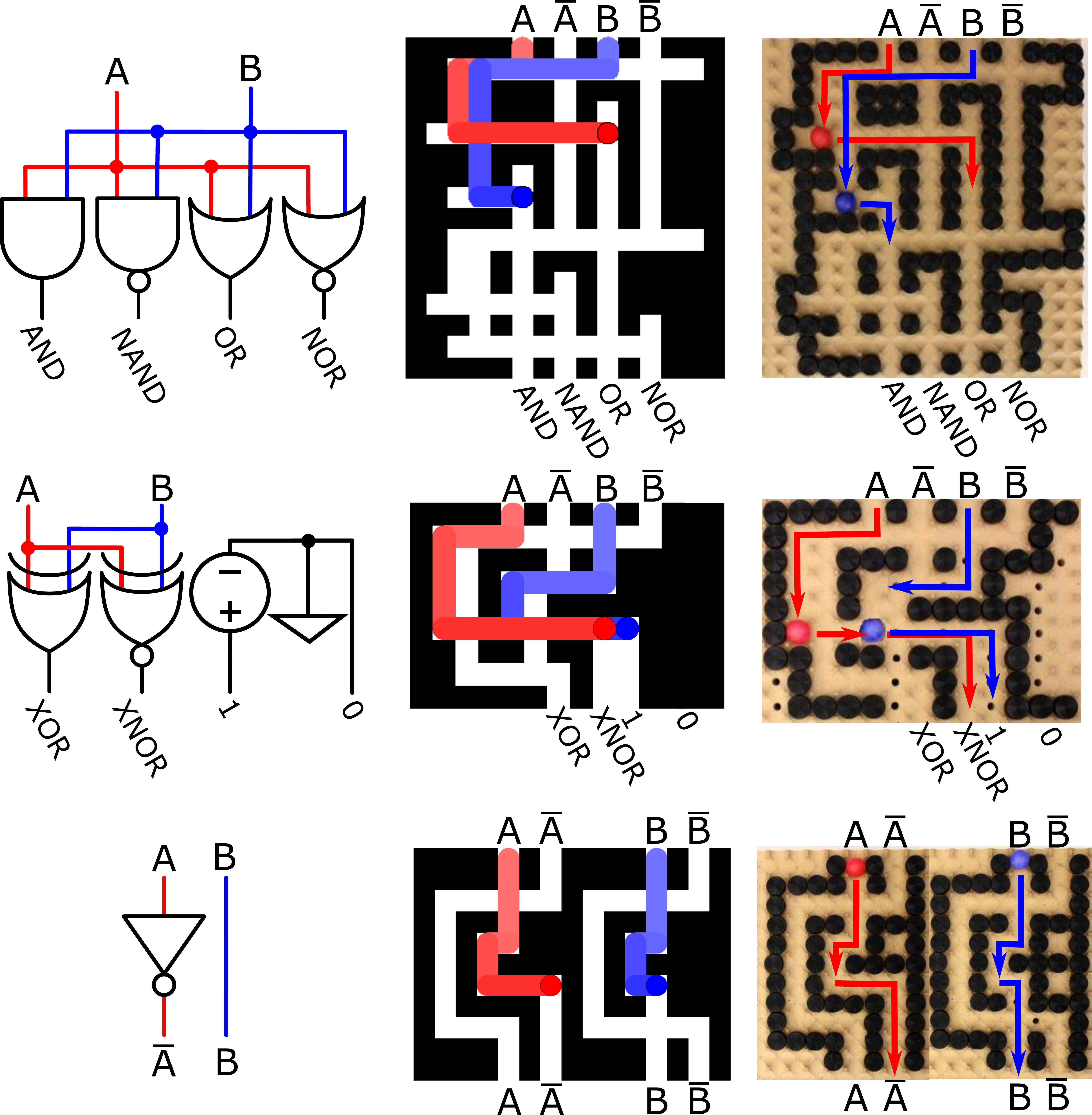}\end{overpic}
\caption{\label{fig:ParticleLogic11}Schematic, diagram, and physical implementation of  dual-rail logic gates. Each gate employs the same clock sequence $\langle d,l,d,r \rangle$, the four inputs correspond to $A,\bar{A},B,\bar{B}$, and the inputs are [1,1]. The top row is a universal logic gate whose  four outputs are {\sc and, nand, or, nor}.  With input [1,1]  the  {\sc and} and {\sc or} outputs are set high. The middle row gate outputs the {\sc xor, xnor} of the inputs and constants 1 and 0.  The bottom row is a {\sc not} gate and a connector. \href{http://youtu.be/mJWl-Pgfos0}{See the attached video at \url{http://youtu.be/mJWl-Pgfos0} for a hardware demonstration.}
}
\vspace{-1em}
\end{figure}

 We study this problem on a two-dimensional grid.  We assume that robots cannot be individually
controlled, but are all simultaneously given a message to travel
 in a given direction until they collide with an obstacle or another robot. 
 This assumption corresponds to
situations with limited-state feedback, or for robots that move at
unpredictable speeds.
Problems of this type are similar to sliding-block puzzles with fixed obstacles
\cite{Demaine2000,Hoffmann2000,Robert-A.-Hearn2002,Holzer2004},
except that all robots receive the same control inputs.


\subsection{Problem Definition}\label{subsec:GeneralProblemDefinition}
More precisely, we consider the following scenario, which we call {\sc GlobalControl-ManyRobots}:
    \begin{enumerate}
\item Initially, the planar square grid is filled with some unit-square robots (each occupying one cell of the grid)  and some fixed unit-square blocks.
\item All robots are commanded in unison: the valid commands are  ``Go Up" ($u$), ``Go Right" ($r$), ``Go Down" ($d$), or ``Go Left" ($l$).  The robots all move in the commanded direction until they hit an obstacle or another robot.  A representative command sequence is $\langle u,r,d,l,d,r,u,\ldots\rangle$. We call these global commands \emph{force-field moves}. We assume we know the maximum dimension of the workspace and issue each command long enough for the robots to reach their maximum extent.
\item The goal is to get
each robot to its specified position.
\end{enumerate}
   The algorithmic decision problem {\sc GlobalControl-ManyRobots}  is to decide whether a given configuration is solvable.
This problem is computationally difficult: we prove PSPACE-completeness in Section~\ref{sec:pspace}. While this result shows
the richness of our model (despite the limited control over the individual parts), it also constitutes
a major impediment for constructive algorithmic work.

This makes developing algorithmic tools that enable global control by uniform commands important. In
Sections~\ref{sec:matrixPermutation} and \ref{sec:logic}, we develop several positive results. The underlying idea is to construct artificial
obstacles (such as walls) that allow arbitrary rearrangements of a given two-dimensional robot swarm.

 Our paper is organized as follows.  After a discussion of related work in Section~\ref{sec:RelatedWork}, we describe how to arrange obstacles to encode matrix permutations in Section~\ref{sec:matrixPermutation}. This result allows us to create useful devices including absolute encoders and matrix displays.  Arbitrary matrix permutations also provides the machinery needed for our result on the problem complexity in Section~\ref{sec:pspace}. In Section \ref{sec:logic} we describe how to implement Boolean algebra, which is enabled by using dual-rail logic. We present our hardware implementation for both matrix permutations and Boolean algebra in Section~\ref{sec:hardware}, and end with concluding remarks in Section~\ref{sec:conclusions}. \href{http://www.mathworks.com/matlabcentral/fileexchange/42892}{All code is available online}~\cite{Becker2013i,Becker2014d}.

%
%
%


\section{Related Work}\label{sec:RelatedWork}

One recent development is the ability to design, produce, and control robots at the micro and nanoscale.
These mobile robots allow a wide range
of possible applications, e.g., targeted drug delivery, micro and nanoscale construction, and Lab-on-a-Chip test devices.
Because (1) the physics of motion at low Reynold's number nanoscale environments requires overcoming a considerable amount of resistance,
and (2) capacity to store energy for computation, communication and motion control shrinks
with the third power of object size, it is clear that classical approaches based on individual motion control cannot be applied.

Instead of individual actuation, a global field is used to control many small agents. An example
is using the global magnetic field from an MRI to guide magneto-tactic bacteria
through a vascular network to deliver payloads at specific locations
\cite{Chanu2008}, and recent work using electromagnets to steer a
magneto-tactic bacterium through a micro-fabricated maze \cite{Khalil2013b}.

\subsubsection{Large Robot Populations}
Due to the efforts of roboticists,  biologists,  and chemists  (e.g. \cite{Rubenstein2012,Ou2013,Chiang2011}),
it is now possible to make and field very large ($10^3$--$10^{14}$) populations of simple robots.   With large populations come two fundamental challenges: (1) how to perform state estimation for the robots, and (2) how to control these robots.

Traditional approaches often assume independent control signals for each robot, but each additional independent signal requires bandwidth and engineering. These bandwidth requirements grow at $O(n)$.
Using independent signals becomes more challenging as the robot size decreases. At the molecular scale, there is a bounded number of modifications that can be made.
  Especially at the micro- and nanoscales it is not practical to encode autonomy in the robots.  Instead, the robots are controlled and interacted with using global control signals.

More recently, robots have been constructed with physical heterogeneity so that they respond differently to a global, broadcast control signal.  Examples include \emph{scratch-drive microrobots}, actuated and controlled by a DC voltage signal from a substrate \cite{Donald2013};   magnetic structures  with different cross-sections that could be independently steered \cite{Floyd2011};   \emph{MagMite} micro-robots with different resonant frequencies and a global magnetic field \cite{Frutiger2008}; and  magnetically controlled nanoscale helical screws constructed to stop movement at different cutoff frequencies of a global magnetic field
\cite{Peyer2013}.

This paper takes a different approach.  We assume a population of approximately identical planar robots (which could be small particles) and  one global control signal that contains the direction all robots should move.  In an open environment, this system is not controllable because the robots move uniformly---implementing any control signal translates the entire group identically.  However, an obstacle-filled workspace allows us to break symmetry. We showed that if we can command the robots to move one unit distance at a time, some goal configurations have easy solutions~\cite{Becker2013b}. Given a large free space, we have an algorithm showing that a single obstacle is sufficient for position control of $N$ robots (video of position control: \url{http://youtu.be/5p_XIad5-Cw}).  However, this result required incremental position control of the group of robots, i.e. the ability to advance them a uniform fixed distance.  This is a strong assumption, and one that we relax in this work.

\subsubsection{Computational Particles}
Amorphous computing \cite{Abelson2007} studies how \emph{computational particles} distributed on a surface can be used to produce computational engines. In a similar manner, we show how to construct logic gates to perform computation in our system, when activated by a global signal.

Another related area of research is Single Instruction Multiple Data (SIMD)
parallel algorithms \cite{Leighton1991}.  In this model, multiple processors
are all fed the same instructions to execute, but they do so on different
data.  This model has some flexibility, for example allowing command execution selectively only on
certain processors and no operations (NOPs) on the remaining processors. 

Our model is actually more extreme: the robots all respond in effectively
the same way to the same instruction.  The only difference is their location,
and which obstacles or robots will thus block them.  In some sense,
our model is essentially Single Instruction, Single Data, Multiple Location.




\subsubsection{Computational Geometry: Robot Box-Pushing}
Many variations of block-pushing puzzles have been explored from a computational complexity viewpoint, with a seminal paper proving NP-hardness by Gordon Wilfong in 1991~\cite{Wilfong1991}.
 The general case of motion-planning when each command moves robots a single unit in a world composed of even a single robot and both \emph{fixed} and \emph{moveable} squares is in the complexity class  PSPACE-complete~\cite{Dor1999}.

The ``move to maximal extent'' motion model we employ is motivated by physical realities where, due to uncertainties in sensing, control application, and robot models, precise quantified movements in a specified direction is not possible, but the input can be applied for a long period of time and be guaranteed that the robots will move to their fullest extent. Lewis uses this model to reduce uncertainty in state estimation\cite{Lewis01092013}.  Maximal extent movement is common in games, including \emph{Ricochet Robots} \cite{Engels2005}, \emph{Atomix} \cite{Holzer2004}, and \emph{PushPush} \cite{Demaine2000}.  In these games the robots move to their full extent with each input, but each robot can be actuated individually.  The complexity of the problem with global inputs to all robots has remained an open problem.



\section{Matrix Permutations}\label{sec:matrixPermutation}

This section investigates a construction problem. Given the {\sc GlobalControl-ManyRobots}  constraints in \ref{subsec:GeneralProblemDefinition},
 what types of control are possible and economical if we are free to design the environment?

First, we describe an arrangement of obstacles that implement an arbitrary
matrix permutation in four commands.  Then we provide efficient algorithms for
sorting matrices, and finish with potential applications.

\subsection{Designing Workspace for a Single Permutation}

A \emph{matrix} is a 2D array of robots
(each possibly a different color).
For an $a_r \times a_c$ matrix $A$ and a $b_r \times b_c$ matrix $B$,
of equal total size $N$,
a \emph{matrix permutation} assigns each element in $A$
a unique position in~$B$.
Figs.~\ref{fig:MatrixPermuteAI} and \ref{fig:PermutationCycleLengthData}
show constructions that execute matrix permutations of
 size $N=15$ and $100$, respectively.
For simplicity of exposition, we assume henceforth that all matrices are
$n \times n$ squares.

\begin{figure}
\centering
\href{http://www.youtube.com/watch?v=3tJdRrNShXM}{
\begin{overpic}[width=\columnwidth]{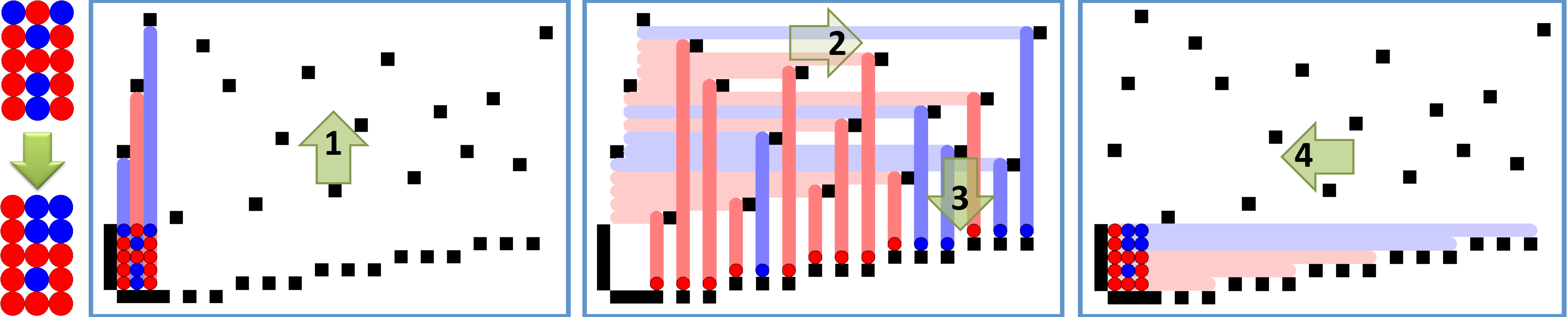}
\end{overpic}}\\
\vspace{1em}
\href{http://www.youtube.com/watch?v=3tJdRrNShXM}{
\begin{overpic}[width=\columnwidth]{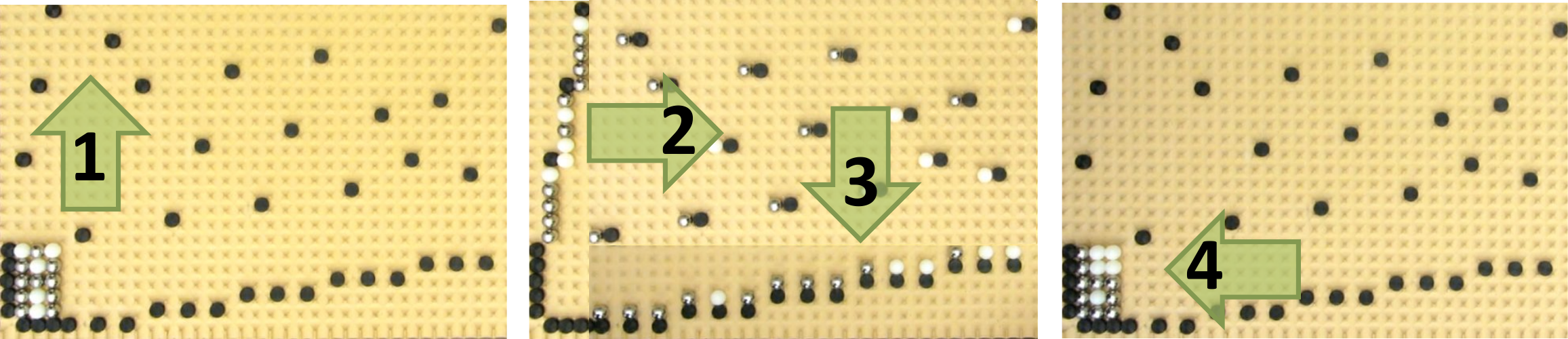}
\end{overpic}}
\caption{\label{fig:MatrixPermuteAI}\href{http://youtu.be/3tJdRrNShXM}{(Top) Matrix permutation for $N$=15. Black cells are obstacles, white cells are free, and colored discs are individual robots. The world has been designed to permute the robots between `A' into `b' every four steps: $\langle u,r,d,l \rangle$. (1) The staggered obstacles on the left spread the matrix vertically, (2) the scattered obstacles on the right permute each element, and (3) the staggered obstacles along the bottom reform each row, which are collected by  (4).
(bottom) Hardware demonstration of a reconfigurable, gravity-fed manipulator that can rearrange (permute) arrays of colored spheres. The demonstration converts `A' to `b', but can be reprogrammed by switching the black stoppers to enable any array permutation.
See video at \url{http://youtu.be/mJWl-Pgfos0}.
}}
\vspace{-1em}
\end{figure}

\begin{theorem} \label{thm:ArbPermutationUsingObstacles}
  Any matrix permutation can be executed by a set of obstacles
  that transforms matrix $A$ into matrix $B$ in just four moves.
  For $N$ robots, the arrangement requires $(3N+1)^2$ space, $4N+1$ obstacles,
  and $10N/v$ time, where $v$ is robot speed in units/s.
\end{theorem}

\begin{proof}
Refer to Figure~\ref{fig:MatrixPermuteAI}
for an example. \href{http://www.mathworks.com/matlabcentral/fileexchange/45538}{{\sc Matlab} code implementing this is available at \cite{Becker2014d}}.  The move sequence is $\langle u,r,d,l \rangle$. We identify the bottom left workspace square as (0,0),  place the bottom-left robot at (1,1), and label the starting configuration $A$ from 1 to $N$ bottom-to-top, left-to-right.  We also assign these indices to the corresponding entries in $B$.

{\bf (Move~1) 
for $i = 1$ to $n$, place an obstacle at ($i,1+n\cdot (i+1)$):}
 We place $n$ obstacles, one for each column, spaced vertically $n$ units apart, such that moving $u$ spreads the robot array into a staggered vertical line. Each robot now has its own row, and are arranged index $1$ to $N$ from bottom to top.
 
{\bf (Move~2)
 for $i = 1$ to $N$, let $[b_r,b_c]$ be the row and column for robot $i$ in $B$. Place an obstacle at ($2(n\cdot b_r + b_c)-(n+1) , n+2i$):}
We place $N$ obstacles to stop each robot during the move~$r$.   Each robot has its own row and can be stopped at any column by its obstacle. We leave an empty column between each obstacle to prevent collisions during the next move.

{\bf (Move~3)
   for $i = 1$ to $N$, place an obstacle at $(n+ 2i-1, \lfloor \frac{i-1}{n} \rfloor )$:}.
 Moving $d$ arranges the robots into their desired rows.  These rows are spread in a staggered horizontal line.
 
{(\bf Move~4)
for $i = 1$ to $n$, place an obstacle at $(0, i)$:}
Moving $l$ stacks the staggered rows into the desired permutation, and returns the array to the initial position.\end{proof}

By reapplying the same permutation enough times, we can return to the original configuration.  The permutation shown in Fig.~\ref{fig:MatrixPermuteAI} returns to the original image in 2 cycles.  For a two-color image, we can always construct a permutation that resets in 2 cycles. We construct an \emph{involution}, a function that is its own inverse, using cycles of length two that transpose two robots. This technique does not extend to images with more than two colors.

\subsection{Physical Absolute Encoders and Animations}
As shown in Fig.~\ref{fig:MatrixPermuteAI}, a permutation gadget allows us to design a display that is hard-coded with a set of pictures.
A potential practical application uses these permutations as a physical absolute encoder or as a pseudo-random number generator. In an \emph{absolute encoder} the current arrangement of robots serves as a unique representation of how many rotations have taken place.
These applications exploit the fact that these physical permutations are cyclic, and that we can design the cycle length.  Applying the CW circular movements $\langle u,r,d,l\rangle$ in succession moves all the robots through one permutation.

The cycle length is the least common multiple of the permutation cycles in the transformation $A\mapsto B$.
Given $N$ robots, we want to partition the set of $k$ permutation cycles in such a way that the sum $\sum_{i=1}^k n_i = N$ and  maximizes $\mathrm{LCM}(n_1,n_2,\ldots,n_k)$.

This cycle length grows rapidly.  For instance,  using $N=100$ robots, we can partition the robots into cycles of length
\{2, 3, 5, 7, 11, 13, 17, 19, 23\}, see Fig. \ref{fig:MatrixPermuteEncoder}.  The $\mathrm{LCM}$ is 223,092,870. See \cite{Deleglise2012} for a more in-depth look at the growth of the maximum cycle length as a function of $N$.

\begin{figure}[t]

\subfloat[][\label{fig:PermutationCycleLength} Absolute encoder cycles ]{
\begin{overpic}[width=\columnwidth]{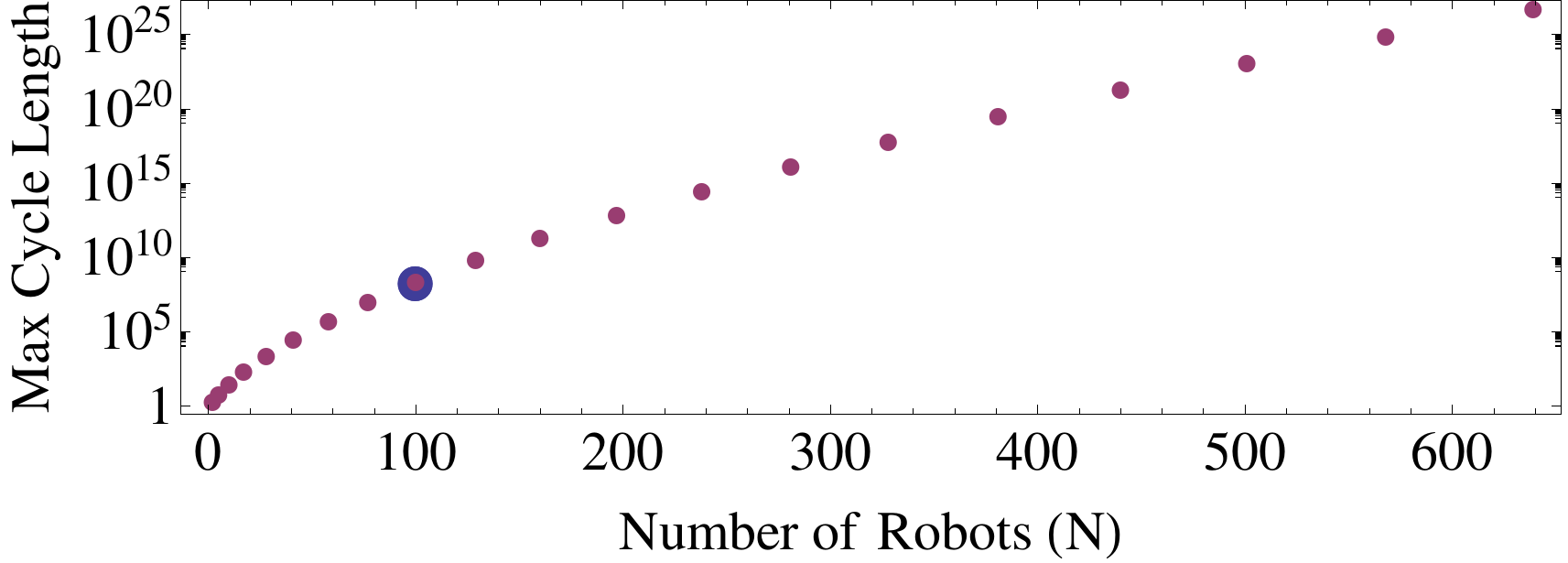}
\end{overpic}
}\\
\subfloat[][\label{fig:MatrixPermuteEncoder} Example encoder]{
\href{http://www.youtube.com/watch?v=eExZO0HrWRQ}{
\begin{overpic}[width=\columnwidth]{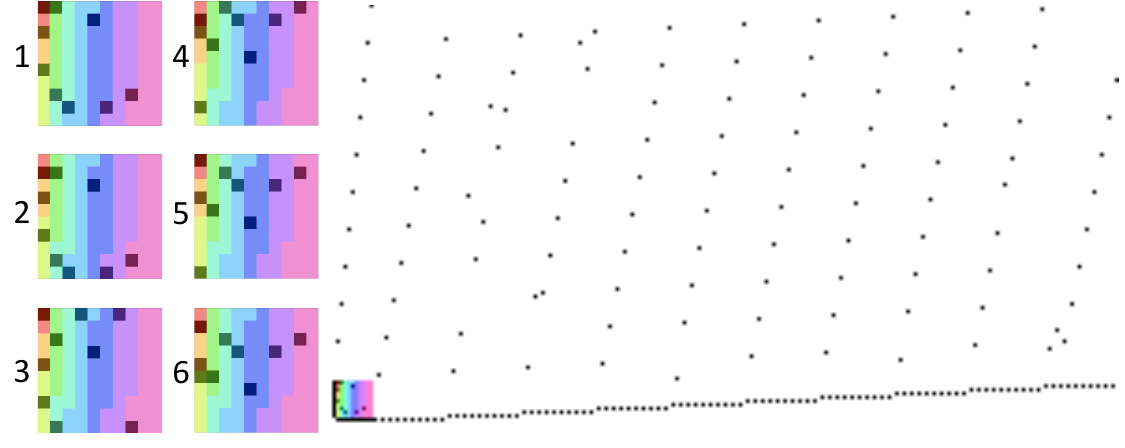}
\end{overpic}}
}
\caption{\label{fig:PermutationCycleLengthData}(a) Using a permutation gadget as an \emph{absolute encoder}. Cycle length increases rapidly as the number of robots increases, and the current arrangement of robots  uniquely represents how many rotations have taken place.
\href{http://www.youtube.com/watch?v=eExZO0HrWRQ}{(b)
An obstacle arrangement to permute a 10$\times$10 matrix in four moves $\langle u,r,d,l\rangle$, with a cycle length over 200 million.  The first 6 permutations are shown at left with each cycle a different color.}
}
\vspace{-1em}
\end{figure}

\subsubsection{Animations}
It would be useful if we could design permutations to generate sequences of images, e.g. $\langle$``R'', ``o'', ``b'', ``o'', ``t''$\rangle$. Surprisingly, there are  sequences of just three images that cannot be constructed with a single permutation. Consider the three 5-robot arrangements
$\square\square\blacksquare\blacksquare\blacksquare$,
$\blacksquare\square\blacksquare\square\blacksquare$,
$\blacksquare\blacksquare\square\blacksquare\square$.  Though permutations between any two exist, there is no single permutation that  can generate all three.
In fact, no single permutation can generate all possible permutations of the given robots.
For the example in Fig.~\ref{fig:MatrixPermuteEncoder}, with 100 robots, 9 painted black and the rest white, the maximum cycle length we can generate is of length $ \approx2\times10^8$, but for permutations of length $N$ with repeated elements $N_1, N_2,
\ldots$, the total number of permutations is
\[
\frac{N!}{N_1! N_2!\ldots N_k!}
\]
For the example above, there are $100!/(9!91!) \approx 2\times10^{12}$ permutations possible.

\subsubsection{Reversible Permutations}
The permutations generators shown in Fig.~\ref{fig:MatrixPermuteEncoder}  are one-way devices.  Attempting to drive them in reverse  $\langle l,d,r,u\rangle$ allows some robots to escape the obstacle region.  It is possible to insert additional obstacles to encode an arbitrary permutation when run in reverse, at a cost of $2N$ additional obstacles and requiring an area  in worst case $3N\times3N$ rather than $N\times2N$.  An example is shown in Fig.~\ref{fig:BubbleSort}. Here, we encode the base permutation $p=(1,2)$ in the CW direction $\langle u,r,d,l \rangle$ and $q=(1,2,\cdots N)$ in the CCW direction $\langle r,u,l,d \rangle$. Repeated application of these two base permutations can generate any permutation, when used in a manner similar to {\sc Bubble Sort}.

%
\begin{figure}[t]
\centering
\begin{overpic}[width=\columnwidth]{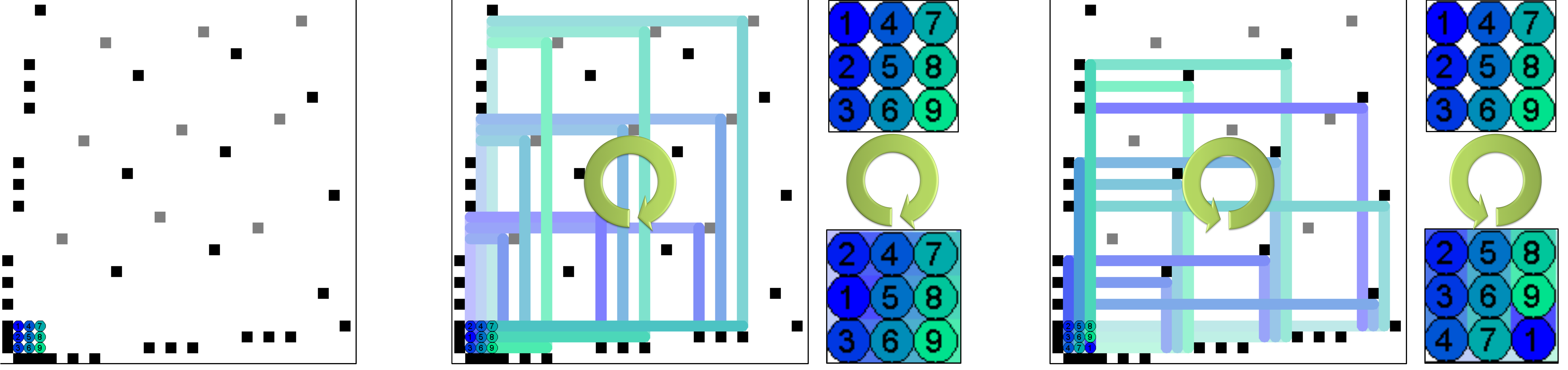}
\end{overpic}
\caption{\label{fig:BubbleSort}The obstacles above generate the base permutation $p=(1,2)$ in the CW direction $\langle u,r,d,l \rangle$ and $q=(1,2,\cdots N)$ in the CCW direction $\langle r,u,l,d \rangle$.
These can be applied repeatedly to {\sc Bubble Sort} the matrix and generate any desired permutation.
}
\vspace{-1em}
\end{figure}

\subsection{Designing a Workspace for Arbitrary Permutations}

There are various ways in which we can exploit Theorem~\ref{thm:ArbPermutationUsingObstacles} in order to generate larger sets
of (or even all) possible permutations. There is a tradeoff between the number
of introduced obstacles and the number of moves required for realizing a permutation. We quote these theorems from \cite{Becker2014a}, as they will be used in our PSPACE-proof. We start with obstacle sets that require only a few moves.

\begin{lemma}\label{lemma:TwoPermutationsGeneratesUniversalPermutations}
      Any permutation of $N$ objects can be generated by the
      two base permutations $p=(1,2)$ and $q=(1,2,\cdots N)$.
      Moreover, any permutation can be generated by a sequence
      of length at most $N^2$ that consists of $p$ and $q$.
\end{lemma}
\begin{proof}
      See Fig.~\ref{fig:BubbleSort}.
      Similar to {\sc Bubble Sort}, we use two nested loops of~$N$.
      Each move consists of performing $q$ once, and $p$ when appropriate.
\end{proof}

This allows us to establish the following result.

\begin{theorem}
\label{thm:NsqMovesToSort}
      We can construct a set of $O(N)$ obstacles such that
      any $n\times n$ arrangement of $N$ pixels can be rearranged into
      any other $n \times n$ arrangement $\pi$ of the same pixels, using
      at most $O(N^2)$ force-field moves.
\end{theorem}


\section{Complexity}\label{sec:pspace}

In previous work \cite{Becker2014a}, we showed that the problem {\sc GlobalControl-ManyRobots} is com\-pu\-ta\-tio\-nal\-ly intractable in a particular
sense: given an initial configuration of movable robots and fixed obstacles, it is
NP-hard to decide whether any robot can be moved to a specified location. It was left as an important open problem whether an even
stronger hardness result applies. In the following, we resolve this problem by proving PSPACE-completeness.

\begin{theorem}
  {\sc GlobalControl-ManyRobots} is PSPACE-complete:
  given an initial configuration of (labeled) movable robots and fixed obstacles, it is
  PSPACE-complete to compute a shortest sequence of force-field moves to achieve another (labeled) configuration.
\end{theorem}

\begin{proof}
The proof is largely based on a complexity result by Jerrum~\cite{j-cfmlg-85}, who considered the following problem:
Given a permutation group, specified by a set of generators, and a single target permutation $\pi$ which is
a member of the group, what is the shortest expression for the target permutation in terms of the generator? This problem
was shown in \cite{j-cfmlg-85} to be PSPACE-complete, even when the generator set consists of only two permutations, say, $\pi_1$ and $\pi_2$.

As shown in the previous Section~III, we can realize any matrix permutation $\pi_i$ of a square arrangement of
robots by a set of obstacles, such that this permutation $\pi_i$ is carried out by a quadruple of force-field moves.
We can combine the sets of obstacles for the two different permutations $\pi_1$ and $\pi_2$, such that $\pi_1$
is realized by going through a clockwise sequence $\langle u, r, d, l\rangle$, while $\pi_2$ is realized by a counterclockwise
sequence $\langle r, u, l, d\rangle$. We now argue that a target permutation $\pi$ of the matrix can be realized by
a minimum-length sequence of $m$ force-field moves, if and only if $\pi$ can be decomposed into a sequence of
a total of $n$ applications of permutations $\pi_1$ and $\pi_2$, where $m=4n$.

The ``if'' part is easy: simply carry out the sequence of $n$ permutations, each realized by a (clockwise or counterclockwise)
quadruple of force-field moves. For the ``only if'' part, suppose we have a shortest sequence of $m$ force-field moves to achieve permutation
$\pi$, and consider an arbitrary subsequence that starts from the base position in which the robots form a square arrangement
in the lower left-hand corner. It is easy to see that a minimum-length sequence cannot contain two consecutive moves that are
both horizontal or both vertical: these moves would have to be be in opposite directions, and we could shorten the sequence by omitting
the first move.
Furthermore, by construction of the obstacle set, the first move must be $u$ or $r$. Now it is easy to check that
the choice of the first move determines the next three ones: $u$ must be followed by $\langle r, d, l\rangle$; similarly,
$r$ must be followed by $\langle u, l, d\rangle$. Any other choice for moves 2--4 would produce a longer overall sequence,
or destroy the matrix by leading to an arrangement from which no recovery to a square matrix is possible. Therefore, the overall sequence
can be decomposed into $m=4n$ clockwise or counterclockwise quadruples. As described, each of these quadruples represents either
$\pi_1$ or $\pi_2$, so $\pi$ can be decomposed into $n$ applications of permutations $\pi_1$ and $\pi_2$.
This completes the proof.
\end{proof}

Note that the result also implies the existence of solutions of exponential length, which can occur with polynomial space.
Binary counters are particular examples of such long sequences that are useful for many purposes.


\section{Particle Logic}\label{sec:logic}

\begin{figure}[]
\centering
\subfloat[][\label{fig:VariableGadget1}
 $i=1$]
{\begin{overpic}[height=0.15\columnwidth]{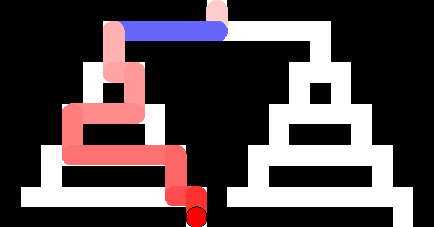}
\end{overpic}
}
\hspace{.1em}
\subfloat[][\label{fig:VariableGadget2}
  $i=2$]
{\begin{overpic}[height=0.15\columnwidth]{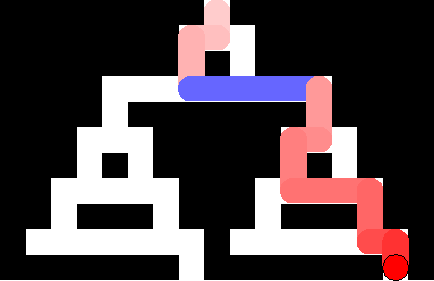}
\end{overpic}
}
\hspace{.1em}
\subfloat[][\label{fig:VariableGadget3}
  $i=3$]
{\begin{overpic}[height=0.15\columnwidth]{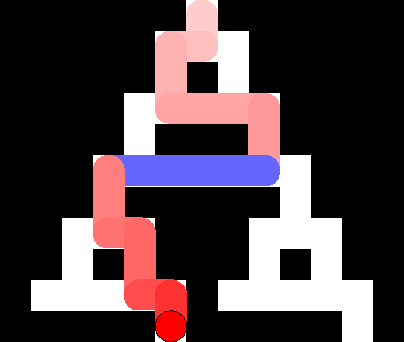}
\end{overpic}
}
\hspace{.1em}
\subfloat[][\label{fig:VariableGadget4}
  $i=4$]
{\begin{overpic}[height=0.15\columnwidth]{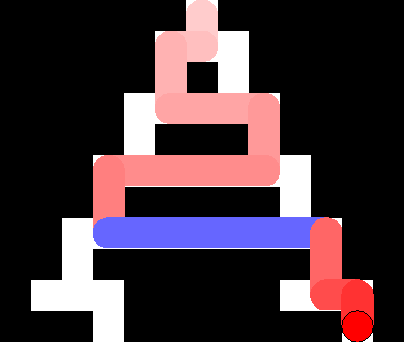}
\end{overpic}
}
\caption{
\label{fig:VariableGadget}
  Variable gadgets that execute by a sequence of $\langle d, l/r \rangle$ moves. The $i$th $l/r$ choice sets the variable to true or false by putting the robot in a separate column. This selection move is shown in blue. Each gadget responds to the $i$th choice but ignores all others, letting us make several copies of the same variable by making multiple gadgets with the same $i$. Above $n$=4, and the input  $\langle d,l,d,r,d,l,d,r,d,r,d\rangle$ causes $i=(1,2,3,4)$ to produce (true, false, true, false). Robots arrive at their output ports at exactly the same time. }
  \vspace{-1em}
\end{figure}

%
%
%

 In our previous work~\cite{Becker2014a}  we showed that with only fixed obstacles and robots that move maximally in response to an input, we can construct a variety of logic elements.  These include variable gadgets that enable setting multiple copies of up to $n$ variables to be true or false, (Fig.~\ref{fig:VariableGadget}),
   $m$-input {\sc or}, and {\sc and} gates. 
   Unfortunately, we cannot build  {\sc not} gates because our system of robots and obstacles is conservative---we cannot create a new robot at the output when no robot is supplied to the input. A  {\sc not}  gate is necessary to construct a logically complete set of gates.  To do this, we rely on a form of \emph{dual-rail logic}, where both the state  and inverse ($A$ and $\bar{A}$) of each signal are propagated throughout the computation.  Dual-rail logic is often used in low-power electronics to increase the signal to noise ratio without increasing the voltage \cite{zimmermann1997low}.  With dual-rail logic we can now construct the missing  {\sc not} gate, as shown in Figs.~\ref{fig:reversibleNOT} and \ref{fig:onewayNOT}. The command sequence $\langle d,l,d,r\rangle$ inverts the input.  By adding one-way valves we can ignore any superfluous commands.  Note that regardless of the command sequence, all robots arrive at their output ports at exactly the same time.

  \begin{figure}
\renewcommand{\figwid}{0.23\columnwidth}
\centering
\scriptsize
\subfloat[][\label{fig:reversibleNOT} 
\\  \centering reversible {\sc not}]
{
\begin{overpic}[width=\figwid]{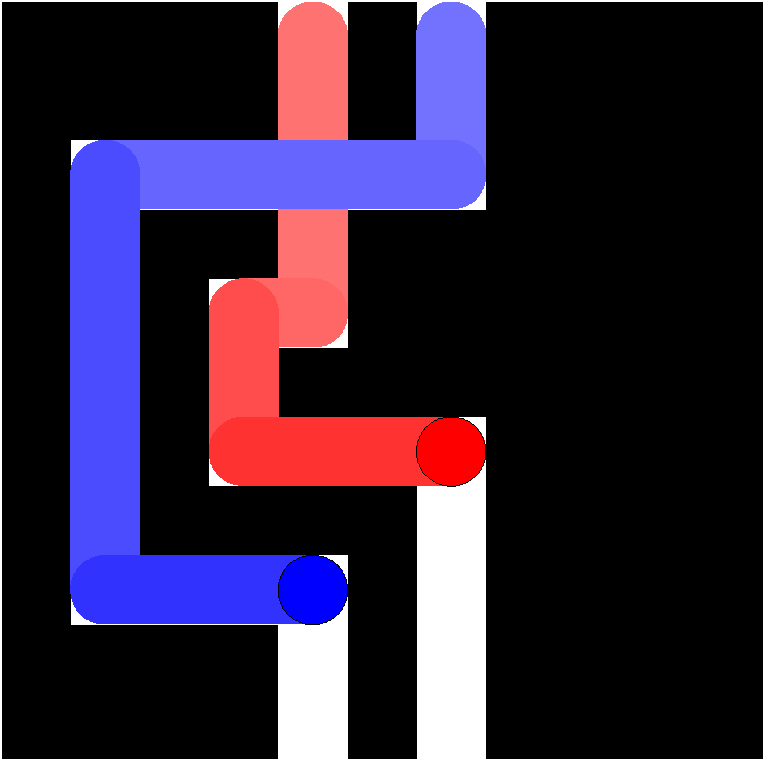}
\end{overpic}
}
\subfloat[][\label{fig:onewayNOT}
\\one-way  {\sc not}]
{
\begin{overpic}[width=\figwid]{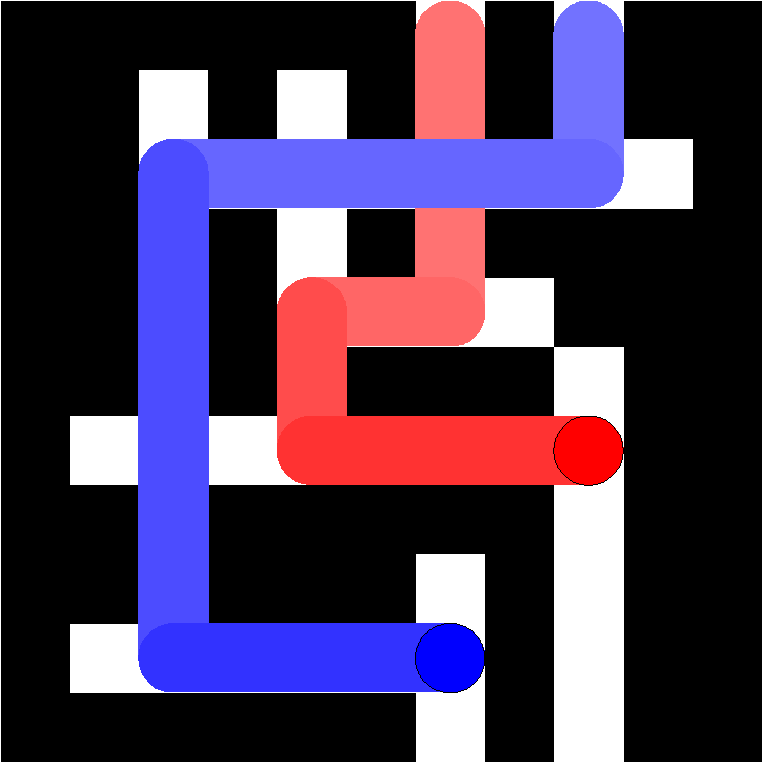}
\end{overpic}
}
\subfloat[][\label{fig:reversibleConnect} 
reversible\\ \centering connect]
{
\centering
\begin{overpic}[width=\figwid]{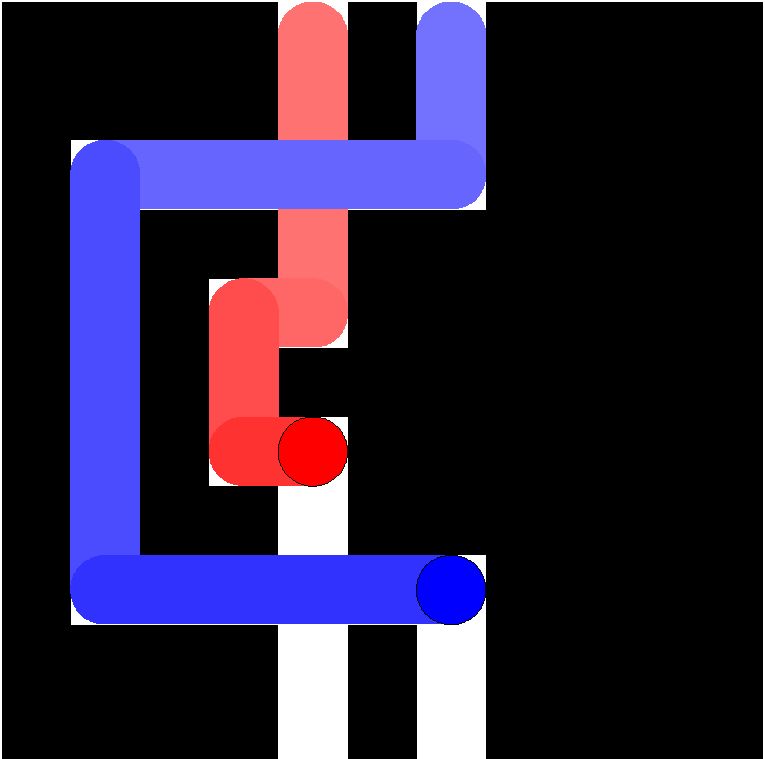}
\end{overpic}
}
\centering
\subfloat[][\label{fig:onewayConnect}
\\ \centering one-way connect]
{
\centering
\begin{overpic}[width=\figwid]{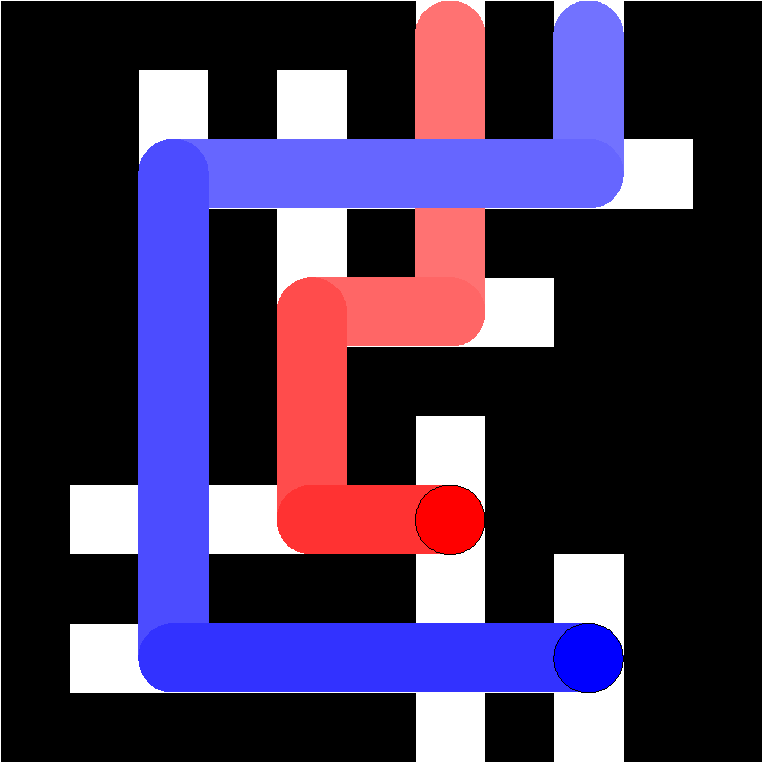}
\end{overpic}
}\\
\subfloat[][\label{fig:ANDdiagram}
 Dual-rail gadget with outputs \{{\sc and,nand,or,nor}\}.]
{\begin{overpic}[width=\columnwidth]{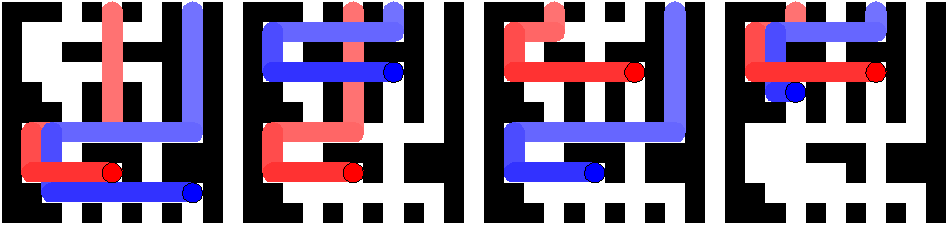}
\scriptsize
\put(7,26){A=0,~~B=0}
\put(7,-1){$0,~1,~0,~1$}
\put(32,26){A=0,~~B=1}
\put(32,-1){$0,~1,~1,~0$}
\put(57.8,26){A=1,~~B=0}
\put(57.8,-1){$0,~1,~1,~0$}
\put(83,26){A=1,~~B=1}
\put(83,-1){$1,~0,~1,~0$}
\end{overpic}
}\\
\subfloat[][\label{fig:XORdiagram}  Dual-rail gadget with outputs \{{\sc xor}(AB), {\sc xnor}(AB), 1\}. ]
{\begin{overpic}[width=\columnwidth]{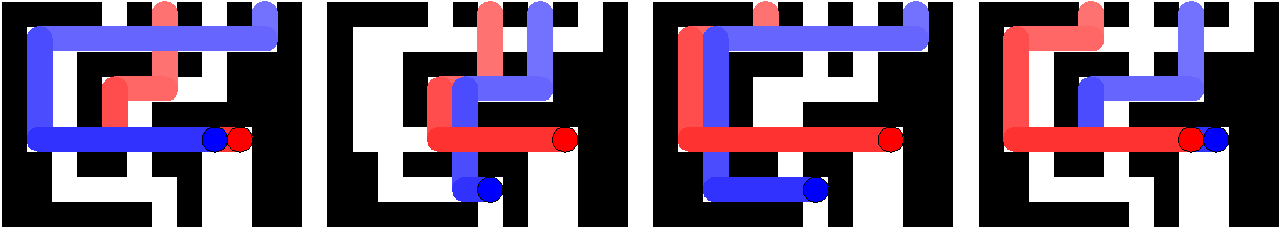}
\scriptsize
\put(7.8,20){A=0,~~B=0}
\put(11.8,-1){$0,1,1$}
\put(33,20){A=0,~~B=1}
\put(37.8,-1){$1,0,1$}
\put(58.8,20){A=1,~~B=0}
\put(62.8,-1){$1,0,1$}
\put(84,20){A=1,~~B=1}
\put(88.8,-1){$0,1,1$}
\end{overpic}
}
\caption{\label{fig:XORandMULTIdiagram}
 Dual-rail gadgets using cycle $\langle d,l,d,r\rangle$.}
 \vspace{-2em}
\end{figure}
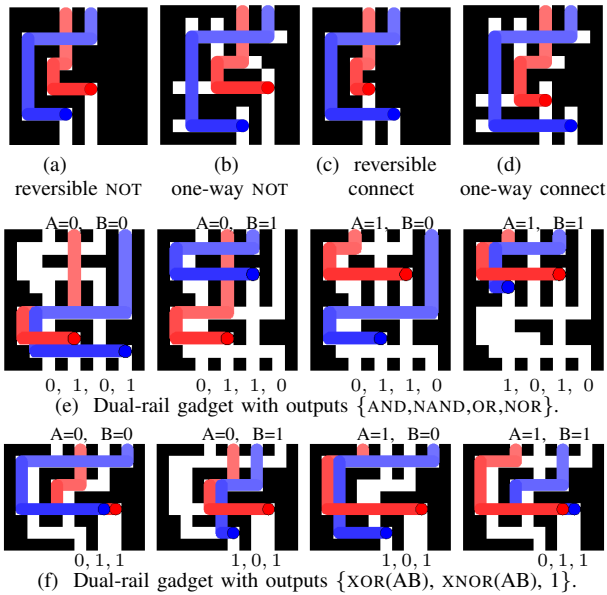

 We now revisit the  {\sc or} and {\sc and} gates of~\cite{Becker2014a} using dual-rail logic and the four inputs $A,\bar{A},B,\bar{B}$.  Surprisingly,  with the gate in Fig.~\ref{fig:ANDdiagram} we can simultaneously compute  {\sc and, nand, or} and {\sc nor}.  using the same  command sequence $\langle d,l,d,r\rangle$ as the {\sc not} gate.  Outputs can be piped for further logic using the interconnections in Figs.~\ref{fig:reversibleConnect} and \ref{fig:onewayConnect}. Unused outputs can be piped into a storage area and recycled for later use.

These gates are reminiscent of the Fredkin gate, a three-bit gate that swaps the last two bits if the first bit is 1~\cite{Fredkin1982ConservativeLogic}.
  They are conservative, in that the number of input and output 1's and 0's are unchanged.  They also form a universal set. Unlike the Fredkin gate, our gate is kinematic rather than dynamic, making it robust to noise and self-synchronizing -- at the end of every move the robots are in a known state, and will not move until we apply another input.  However, unlike the Fredkin gate, our {\sc and/nand/or/nor} gate is not reversible.

  Dual-rail devices open up new opportunities, including  {\sc xor} and {\sc xnor} gates, which are not conservative using single-rail logic.  This gate, shown in~Fig.~\ref{fig:ParticleLogic11} also outputs a constant 1 and 0.

  With an {\sc and} and {\sc xor} we can compactly construct a half-adder.  We are hindered by an inability to construct a fan-out device that produces multiple copies of an input.  Instead, we must take any logical expression and create multiple copies of each input.  For example, a half-adder requires only one {\sc xor} and one  {\sc and} gate, but our particle computation requires two A  and two B inputs.


To make our gate robust to input sequences that deviate from $\langle d, l,d,r \rangle$, we can create caves that act as one-way valves, as shown in Figs.~\ref{fig:onewayNOT} and \ref{fig:onewayConnect}.  After an $l$ input, the robot is at the left end of a horizontal corridor.  By placing a 1-unit cave at the rightmost end of the corridor we can latch the $l$ input---moving right inserts the robot into a cave that can only be exited by an $l$ input.  Similarly, by placing a 1-unit cave above the leftmost end of the corridor, a $u$ input inserts the robot into a cave that can only be exited by an $l$ input.
%
%

\section{hardware demonstrations}\label{sec:hardware}

Fig.~\ref{fig:HardwarePlatformMedium} shows our scale prototype of a reconfigurable  {\sc GlobalControl-ManyRobots} environment, using 1.27 diam steel and nylon bearings as our robots and a naturally-occuring gravity field as the control field.  The prototype is a 61$\times$61 cm square sheet of 2 cm thick medium-density fiberboard (MDF), with a lattice grid of hemispherical-profile, 1.27 cm grooves milled at 1.27 cm spacing in the $x$ and $y$ directions.  At the intersection of each set of orthogonal grooves is a 4 mm diameter hole.  We can then insert plastic-headed thumb screws with 1.27 cm diam heads \href{http://www.mcmaster.com/#91185a309}{(McMaster \#91185A444)} to serve as obstacles.  The prototype is centered and glued on top of a 20$\times$20 cm square section of MDF.  Pushing down on any  top board edge tilts the entire prototype $u,r,l,$ or $d$, and the bearings roll until they hit an obstacle or another bearing.  \href{http://youtu.be/mJWl-Pgfos0}{The companion video illustrates this prototype configured to create a permutation that converts `A' to `b' under the command sequence $\langle u,r,d,l\rangle$}, also shown in Fig.~\ref{fig:MatrixPermuteAI}.

We have also configured the prototype to generate the dual-rail universal Boolean gate in Fig.~\ref{fig:ParticleLogic11} and the logical interconnects of Fig.~\ref{fig:reversibleConnect},  \href{http://youtu.be/mJWl-Pgfos0}{see the accompanying video}.  The long open paths in the permutation arrangement often lead to errors when bearings pop off their proper paths.  The enclosed mazes of the logic gates are more reliable and we have not recorded any errors.

\begin{figure}
\begin{overpic}[width=\columnwidth]{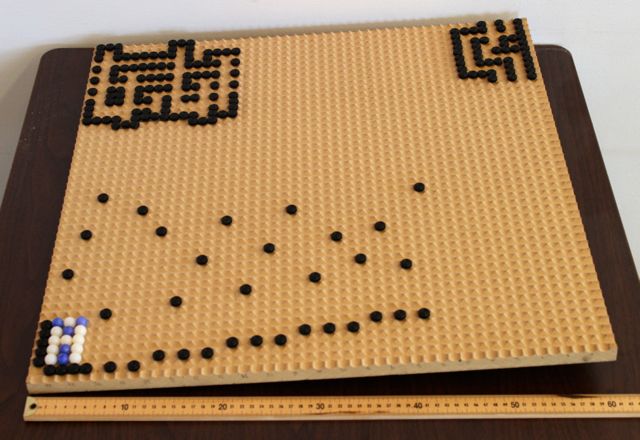}
\end{overpic}
\caption{\label{fig:HardwarePlatformMedium}
Gravity-fed hardware implementation of  {\sc GlobalControl-ManyRobots}.  Bottom left is a matrix permutation for changing `A' to `b', top left is a combination {\sc and, nand, or, nor} gate, and top right is a {\sc not} gate. \href{http://youtu.be/mJWl-Pgfos0}{See \url{http://youtu.be/mJWl-Pgfos0}.} }
 \vspace{-1em}
\end{figure}


\section{Conclusions}\label{sec:conclusions}

We analyzed the problem of steering many robots with uniform inputs in a {\sc 2d} environment containing obstacles. We introduced dual-rail particle logic computation, and designed environments that can efficiently perform matrix operations on groups of robots in parallel---our matrix permutation requires only four moves for any number of robots. These matrix operations enabled us to prove the general motion planning problem PSPACE-complete.

  There remain many interesting problems to solve. We are motivated by
practical challenges in steering micro-robots through vascular networks,  which
are common in biology. Though some are two-dimensional, including the leaf
example in Fig.~\ref{fig:vascularNetwork} and endothelial networks on the
surface of organs, many of these networks are three dimensional.
Magnetically actuated systems are capable of providing 3D control inputs, but
control design poses additional challenges.

   The paper investigated a subset of control in which all robots move
maximally. Future work should investigate more general motion---what happens
to our complexity proof if we can move all the robots a discrete distance, or
along an arbitrary curve? We also abstracted important practical constraints e.g.,
ferromagnetic objects tend to clump in a magnetic field, and most magnetic fields are not perfectly uniform.

Using \emph{dual-rail logic}, we are limited to conservative logic. We cannot
create new robots, so logic such as a multi-bit adder require exponentially
increasing numbers of inputs. Generating fan-out gates seems to require
additional flexibility in our problem definition, because conservation rules are
violated.  Some way of encoding an order of precedence is needed so that a
reversible operation on robot $a$ can affect robot $b$.  Possible approaches use non-unit size components--either 2$\times$1 robots, or 0.5$\times$1 obstacles.

Finally, our research has potential applications in micro-construction and
nano-assembly.  These applications require additional theoretical analysis to
model heterogeneous objects and objects that bond when forced together, e.g.,
MEMS components and molecular chains.



\section*{Acknowledgments}
We acknowledge the helpful discussion and motivating experimental efforts with \emph{T. pyriformis} cells by Yan Ou and Agung Julius at RPI and Paul Kim and MinJun Kim at Drexel University.  Ricardo Marquez and Artie Shen assisted with photography and the hardware platform.
This work was supported by the National Science Foundation under 
\href{http://www.nsf.gov/awardsearch/showAward?AWD_ID=1035716}{CPS-1035716}.   


\bibliographystyle{IEEEtran}
\bibliography{IEEEabrv,tilt}

\end{document}